\documentclass{article}

\usepackage[preprint]{neurips_2025}

\usepackage[utf8]{inputenc} 
\usepackage[T1]{fontenc}    
\usepackage{hyperref}       
\usepackage{url}            
\usepackage{booktabs}       
\usepackage{amsfonts}       
\usepackage{nicefrac}       
\usepackage{microtype}      
\usepackage{xcolor}
\usepackage{enumitem}
\usepackage{todonotes}
\usepackage{mathtools}
\usepackage{xspace}
\usepackage{amsmath}
\usepackage{amssymb}
\usepackage{physics}
\usepackage{dsfont}
\usepackage{amsthm}
\usepackage{stackengine}
\usepackage{dsfont}
\usepackage{mathtools}
\usepackage{multirow}
\usepackage[linesnumbered,titlenumbered,ruled,vlined,resetcount]{algorithm2e}
\usepackage[capitalize]{cleveref}  

\SetKwComment{Comment}{$\triangleright$\ }{}

\theoremstyle{definition}
\newtheorem{lemma}{Lemma}

\newtheorem{theorem}{Theorem}
\newtheorem{remark}{Remark}
\newtheorem{corollary}{Corollary}
\newtheorem{example}{Example}

\newcommand{\X}{\mathbf{X}}
\newcommand{\vomega}{\ensuremath{\vec{\omega}}\xspace}
\newcommand{\vx}{\ensuremath{\vec{x}}\xspace}
\newcommand{\vy}{\ensuremath{\vec{y}}\xspace}
\newcommand{\vr}{\ensuremath{\vec{r}}\xspace}
\newcommand{\vz}{\ensuremath{\vec{z}}\xspace}
\renewcommand{\H}{\ensuremath{\mathcal{H}}\xspace}

\newcommand{\R}{\ensuremath{\mathbb{R}}\xspace}

\newcommand{\MMD}{\ensuremath{\operatorname{MMD}}\xspace}
\newcommand{\MMDM}{\ensuremath{\operatorname{MMDM}}\xspace}
\newcommand*\diff{\mathop{}\!\mathrm{d}}
\newcommand{\dual}{\ensuremath{\vec{\lambda}}\xspace}
\newcommand{\argmin}{\ensuremath{\operatorname{arg\,min}}\xspace}

\newcommand{\RMMD}{\ensuremath{R_{\text{MMD}}}\xspace}
\newcommand{\vlambda}{\ensuremath{\vec{\lambda}}\xspace}
\newcommand{\iu}{\mathrm{i}\mkern1mu}
\newcommand{\algo}{MMDN\xspace}

\title{MMD-Newton Method for Multi-objective Optimization}

\author{%
  Hao Wang\thanks{Hao Wang is also with Leiden applied quantum algorithms (aQa) initiative} \\
  Leiden University\\
  \texttt{h.wang@liacs.leidenuniv.nl} \\
  \And
  Chenyu Shi \\
  Leiden University\\
  \texttt{c.shi@liacs.leidenuniv.nl} \\
  \And
  Angel E. Rodriguez-Fernandez\\
  Cinvestav \\
  \texttt{angel.rodriguez@cinvestav.mx} \\
  \And
  Oliver Sch\"utze\\
  Cinvestav \\
  \texttt{schuetze@cs.cinvestav.mx} \\
}

\begin{document}
\maketitle

\begin{abstract}
Maximum mean discrepancy (MMD) has been widely employed to measure the distance between probability distributions. In this paper, we propose using MMD to solve continuous multi-objective optimization problems (MOPs). For solving MOPs, a common approach is to minimize the distance (e.g., Hausdorff) between a finite approximate set of the Pareto front and a reference set. Viewing these two sets as empirical measures, we propose using MMD to measure the distance between them. To minimize the MMD value, we provide the analytical expression of its gradient and Hessian matrix w.r.t. the search variables, and use them to devise a novel set-oriented, MMD-based Newton (MMDN) method. Also, we analyze the theoretical properties of MMD's gradient and Hessian, including the first-order stationary condition and the eigenspectrum of the Hessian, which are important for verifying the correctness of MMDN. To solve complicated problems, we propose hybridizing MMDN with multiobjective evolutionary algorithms (MOEAs), where we first execute an EA for several iterations to get close to the global Pareto front and then warm-start MMDN with the result of the MOEA to efficiently refine the approximation. We empirically test the hybrid algorithm on 11 widely used benchmark problems, and the results show the hybrid (MMDN + MOEA) can achieve a much better optimization accuracy than EA alone with the same computation budget.
\end{abstract}

\section{Introduction}
Assume a real-valued, black-box multi-objective function $F\colon \R^n \rightarrow \R^k$, and $F\in C^2(\R^n)$. We wish to minimize the following constrained multi-objective problem (CMOP):
\begin{equation}
\label{eq:MOPeq}
\begin{array}{cl}
\min\limits_{\vx\in\R^n} & \; F(\vx)\\
\mbox{s.t.}   & \; h_i(\vx) = 0,\quad i=1,\ldots, p, \\
              & \; g_j(\vx) \leq 0,\quad j=1,\ldots, q, \\
\end{array}
\end{equation}
where we assume the constraints are also $C^2$ functions. The optimal solution (under Pareto order) is typically a sub-manifold of $\R^n$, called the efficient set (which can be disconnected). The image of the efficient set is the Pareto front. In the multi-objective optimization community, there have been many algorithms proposed to solve the above problem, including multi-objective evolutionary algorithms (MOEAs)~\cite{BeumeNE07,ZhangL07,DebJ14}, Bayesian optimization~\cite{EmmerichYD0F16,YangEDB19,0025YA24}, and mathematical optimization methods~\cite{FliegeS00,eichfelder:08,miettinen2012nonlinear,martin:18,Sosa-HernandezS20,sinha:22,schuetze_pt:24,WangRUCS24}. 
Among these algorithms, we focus on mathematical methods since for differentiable functions, such methods converge much faster than MOEAs or Bayesian optimization.
Such methods typically employ a finite approximation to the efficient set, denoted by $X = \{\vx^{\,1}, \vx^{\,2}, \ldots, \vx^{\,\mu}\}\subseteq \R^n$, and its image $Y \coloneqq F[X] = \{\vy^{\,1}, \vy^{\,2}, \ldots, \vy^{\,\mu}\}$ (called Pareto approximation set). Given a rough guess $R\subseteq \R^k$ of the Pareto front, called the reference set, such methods minimize the distance between $Y$ and $R$ up to some metric, e.g., generational distance~\cite{IshibuchiMTN15}, inverted generational distance~\cite{CoelloS04}, and averaged Hausdorff distance~\cite{SchutzeELC12}. However, the performance of such algorithms heavily depends on the quality of $R$, e.g., if $R$ fails to span the entire Pareto front, then the algorithm that minimizes the distance thereto cannot recover the Pareto front fully. 

In this paper, we propose using the \emph{maximum mean discrepancy} (MMD) metric to measure the distance between $Y$ and $R$ due to its nice property: \emph{MMD not only incorporates the discrepancy between $Y$ and $R$, but it also measures the diversity of $Y$}, which can improve the Pareto approximation set even when the reference $R$ is imperfect. We provide an example in~\cref{fig:MMDN-example}: when minimizing MMD, we bring $Y$ onto the Pareto front and we increase the diversity/spread thereof.\\
Our contributions are: 
\begin{enumerate}
    \item We propose a novel, set-oriented, MMD-based Newton method (MMDN) for solving numerical multi-objective optimization problems (\cref{sec:MMD-Newton}).
    \item We provide the analytical expression of MMD's gradient and Hessian matrix. We analyze the first-order stationary condition for MMD's gradient, and we bound the eigenspectrum of its Hessian (\cref{sec:MMD-derivatives}).
    \item We devise a hybridization of MMDN and multi-objective evolutionary algorithms (MOEAs), which can solve complicated problems with much higher accuracy/efficiency (\cref{sec:MMD-EA-hybrid}). We benchmark the hybrid algorithm over three state-of-the-art MOEAs on 11 test problems, and the results show \algo can help improve the performance of MOEA (\cref{sec:experiments}).
\end{enumerate}

\begin{figure}[t]
  \centering
  \includegraphics[width=.85\textwidth, trim=4mm 5mm 130mm 15mm, clip]{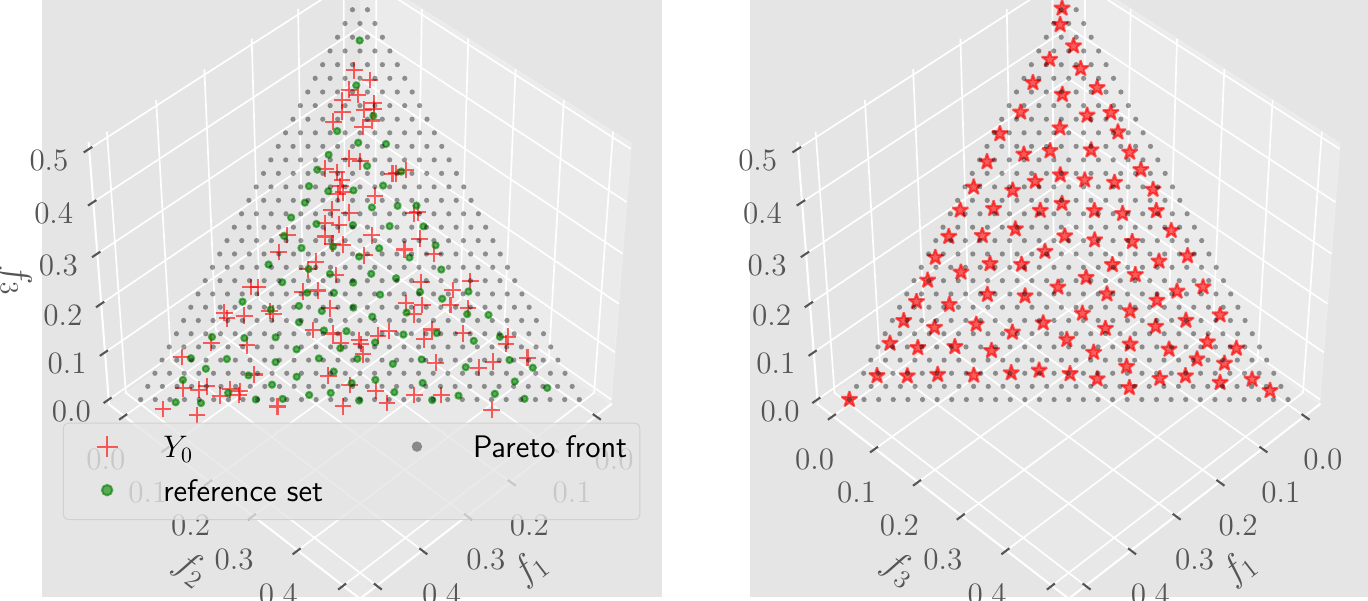}
  \caption{Example of MMD-Newton (\algo) achieving uniform coverage of the Pareto front with imperfect reference set. \textbf{Left:} on the three-objective DTLZ1~\cite{Deb2002DTLZ} problem, we initialize \algo with $Y_0$ and an imperfect reference set $R$ (green points) which does not span the entire Pareto front (the black points). \textbf{Right:} the result of five iterations of \algo (\cref{eq:Newton_step_S}) to minimize the $\MMD^2(Y_0,R)$, which covers the Pareto front uniformly, which is because MMD also maximizes the diversity of $Y_0$.
  \label{fig:MMDN-example}
  }
\end{figure}
This paper is organized as follows: \cref{sec:background} briefly reviews the background of this work. In~\cref{sec:MMD-Newton}, we propose the MMD-Newton method. In~\cref{sec:MMD-derivatives}, we provide MMD's derivatives and investigate their properties. \cref{sec:MMD-EA-hybrid} proposes a hybrid of \algo with MOEAs. In~\cref{sec:experiments}, we present benchmarking results of the hybrid algorithm, followed by some limitations of the work (\cref{sec:limitations}). Finally, \cref{sec:conclusion} concludes the paper and discusses potential future research. 

\section{Background}\label{sec:background}
\paragraph{Numerical multi-objective optimization problem (MOP)} minimizes multiple conflicting objective functions, i.e., $F\colon \mathcal{X}\subseteq \R^n \rightarrow \R^k, F=(f_1, \ldots, f_k), f_i:\mathcal{X}\rightarrow \R, i\in[1..k]$. The \emph{Pareto order/dominance} $\prec$ on $\mathbb{R}^k$ is defined: $\vec{y}^{\,1} \prec \vec{y}^{\,2}$ iff. $y^{1}_i \leq y^{2}_i, i\in[1..k]$ and $\vec{y}^{\,1} \neq \vec{y}^{\,2}$. A point $\vec{x} \in \R^n$ is efficient iff. $\nexists \vec{x}'\in\R^n, F(\vec{x}') \prec F(\vec{x})$. The set of all efficient points of $\mathcal{X}$ is called the \emph{efficient set}. The image of the efficient set under $F$ is called the \emph{Pareto front} - typically a smooth sub-manifold of $\R^k$. The subset of $\R^k$ containing points that dominate at least one point of the Pareto front is called \emph{utopian region}. Multi-objective algorithms employ a finite multiset $X = \{\vec{x}^{\,1}, \dots,\vec{x}^{\,\mu}\}$ to approximate the efficient set. Its image $F(X)$ shall be referred to as \emph{Pareto approximation set}.

\paragraph{RKHS and kernel mean embedding} Consider a positive definite function/kernel $k\colon \R^k\times \R^k \rightarrow \R$, e.g., Gaussian and the Mat\'ern family~\cite{Genton01}. The kernel might be parameterized by a free parameter $\theta$ called length-scale, e.g., for the Gaussian kernel $k(\vy, \vy\,') = \exp\left(-\theta\|\vy - \vy\,'\|_2^2\right)$. To each kernel $k$ function, a feature map $\phi\colon \R^k\rightarrow \H$ is associated by the Moore-Aronszajn theorem, where \H is a reproducing kernel Hilbert space (RKHS) of functions $f\colon \R^k\rightarrow \R$ with the inner product $\langle \phi(\vy), \phi(\vy\,')\rangle_\H = k(\vy, \vy\,')$. The kernel mean embedding (KME) method maps a probability measure $P$ on $\R^k$ to \H~\cite{MuandetFSS17}: $\psi(P) =\int_{\R^k} k(\cdot, \vx) \diff P \in \H$. We call $\psi(P)$ the KME of $P$ with kernel $k$. Consider subsets in the objective space $Y,R\subseteq\R^k$ as empirical measures, i.e., $ \widehat{P}_R (\vec{z}) = \lambda^{-1}\sum_{\vx\in R} \mathds{1}(\vx\preceq\vec{z}), \widehat{P}_Y (\vec{z}) = \mu^{-1}\sum_{\vx\in Y} \mathds{1}(\vx\preceq\vec{z})$, we apply KME to $Y$ and $R$: $\psi(R) = \lambda^{-1}\sum_{\vr\in R} k(\cdot, \vr), \psi(Y) =  \mu^{-1}\sum_{\vy\in Y} k(\cdot, \vy)$.

\paragraph{Maximum mean discrepancy} (MMD)~\cite{GrettonBRSS12,BorgwardtGRKSS06} between empirical measures $Y$ and $R$ is:
\begin{align*} 
    \MMD^2(Y, R)=\norm{\psi(R) - \psi(Y)}_\H^2=\frac{1}{\mu^2} \sum_{i, j} k(\vy^{\,i}, \vy^{\,j}) + \frac{1}{\lambda^2} \sum_{i, j} k(\vr^{\,i}, \vr^{\,j}) - \frac{2}{\mu\lambda} \sum_{i, j}k(\vr^{\,i}, \vy^{\,j}).
\end{align*}
When minimizing $\MMD^2(Y, R)$ over all $Y\subseteq\R^k$ (fixing the reference set $R$), MMD contains effectively two terms: (1) the inner product of $\psi(R)$ and $\psi(Y)$ and (2) the RKHS norm $\norm{\psi(Y)}^2_\H$, where the former brings $Y$ closer to $R$ and the latter maximizes the spread of $Y$, if for instance, a monotonically decreasing kernel is used, e.g., Gaussian $k(\vec{x}, \vec{y}) = \exp(-\theta\lVert\vec{x} - \vec{y}\rVert_2^2)$. This observation is essential to the advantage of the MMD-Newton algorithm proposed in this paper. See~\cref{fig:MMDN-example} for an illustration.\\
If we use a bounded stationary kernel $k$, i.e., $k(\vy, \vy\,') = \tilde{k}(\vy - \vy\,')$ for some positive function $\tilde{k}$ and $\sup_{\vy} |\tilde{k}(\vy)| < \infty$, then by the Bochner theorem~\cite{bochner1933monotone}, MMD admits the spectral representation~\cite{MuandetFSS17}:
\begin{align} \label{eq:spectrum-MMD}
    \MMD^2(Y, R) = \int_{\R^k} \left\vert \widehat{\varphi}_Y  - \widehat{\varphi}_R \right\vert^2 G\left(\diff \vomega\right), 
\end{align}
where $\widehat{\varphi}_Y(\vomega) = \frac{1}{\mu} \sum_{\alpha=1}^{\mu} \exp\left(\iu \langle\vomega, \vy^{\,\alpha}\rangle\right), \widehat{\varphi}_R(\vomega) = \frac{1}{\lambda} \sum_{\alpha=1}^{\lambda} \exp\left(\iu \langle\vomega, \vr^{\,\alpha}\rangle\right)$ are empirical characteristic function of $Y$ and $R$, respectively, and 
$G$ is the non-negative Borel measure on $\R^k$ associated to the kernel $k$. Recently, MMD has been proposed as a performance indicator to evaluate the quality of Pareto front approximations~\cite{CaiXLSXLI22}. It is well-known that MMD is bounded above by the Wasserstein-1 and total variation metric~\cite{SriperumbudurGFSL10,VayerG23}. In general, it remains an open question how MMD is related to other metrics in the multi-objective optimization community, e.g., average Hausdorff distance~\cite{SchutzeELC12}.

\section{MMD-Newton Method} \label{sec:MMD-Newton}
\paragraph{Set-oriented approach for multi-objective optimization} We stress that a set $X = \{\vx^{\,1},\ldots, \vx^{\,\mu}\} \subset \R^n$ of 
$\mu$ candidate solutions in $\mathbb{R}^n$ can essentially be considered by 
a point $\X \in \mathbb{R}^{\mu n}$, i.e., $\X = (x^{1}_1,\ldots, x^{1}_n, x^{2}_1,\ldots, x^{2}_n, \ldots, x^{\mu}_1,\ldots, x^{\mu}_n) \in \R^{\mu n}$. Denote by $Y=F(X)$ the image of $X$. We consider MMD as a function $\X$, $\MMD^2(\X, R)\coloneqq \MMD^2(F[X], R)$. We denote by $\nabla \MMD^2(\X, R) \in \R^{\mu n}$ and $\nabla^2 \MMD^2(\X, R) \in \R^{\mu n \times \mu n}$ the the gradient and the Hessian of MMD at $\X$, respectively (see~\cref{eq:MMD-gradient} and~\cref{eq:MMD-Hessian} in~\cref{sec:MMD-derivatives}). Using $\X$, the equality constraint w.r.t.~\cref{eq:MOPeq} becomes 
\begin{equation}
 h_i(\vx^{\,j}) = 0,\qquad  i = 1,\ldots, p, \; j = 1,\ldots, \mu. 
\end{equation}
To ease the discussion, we hence define, for $i=1,\ldots, p$
\begin{equation}
\bar{h}^i\colon\mathbb{R}^{\mu n} \to \R^p, \qquad \bar{h}^i(\X) = \left(h_1(\vx^{\,i}), h_2(\vx^{\,i}),\ldots, h_p(\vx^{\,i})\right)^\top,
\end{equation}
and the overall constraint function $\bar{h}:\mathbb{R}^{\mu n}\to\mathbb{R}^{\mu p }, \bar{h}(\X) \coloneqq \left(\bar{h}^1(\X), \bar{h}^2(\X), \ldots, \bar{h}^p(\X)\right)$.
Let the functions $\bar{g}(\X)\coloneqq \left(\bar{g}^1(\X), \bar{g}^2(\X), \ldots, \bar{g}^q(\X)\right)$ be defined analogously for the inequalities, which we handle with the \emph{active set} method~\cite{WangRUCS24}: an inequality constraint $\bar{g}^\ell$ is considered active if $\bar{g}^\ell(\X) > -\operatorname{tol}$ for some small $\operatorname{tol}>0$. We treat the active ones as equality constraints. Denote the active set by $A(\X)=\{\ell\in[1..q]\colon \bar{g}^\ell(\X) > -\operatorname{tol}\}$. The set-oriented optimization problem is:
\begin{equation}
\label{eq:MMDmin}
\begin{array}{cl}
\min\limits_{\X\in\mathbb{R}^{\mu n}} & \MMD^2(\X, R) \\
\mbox{s.t.} & \bar{h}(\X) = 0\\
& \bar{g}^\ell(\X) = 0, \quad \ell \in A(\X)\\
\end{array}
\end{equation}
\paragraph{Deriving the Newton method}
In the following, we derive the Newton method for equality-constrained problems. For
the treatment of inequalities, we use an active set method similar to the one described
in \cite{WangRUCS24}. At every point $\X$, we are hence dealing with a particular
equality-constrained MOP. The treatment of active inequalities is exactly the same, which we omit here. The derivative/Jacobian of $\bar{h}$ at $\X$ is given by 
\begin{equation}
J \coloneqq \nabla \bar{h}(\X) = \left(\nabla \bar{h}^1(\X)^\top, \ldots, \nabla \bar{h}^p(\X)^\top\right)^\top=
 \operatorname{diag}\left(J^1, J^2\ldots, J^{\mu}\right) \in \mathbb{R}^{\mu p \times \mu n},
\end{equation}
where $J^i$ denotes the Jacobian of $h$ at $\vx^{\,i}$, i.e., $J^i = (\nabla h_1(\vx^{\,i}), \nabla h_2(\vx^{\,i}), \ldots, \nabla h_p (\vx^{\,i}))^\top \in \mathbb{R}^{p\times n}$.

In order to derive the Newton methods for this problem, we first consider the Karush-Kuhn-Tucker (KKT) equations of \cref{eq:MMDmin} with Lagrange multipliers $\lambda \in\mathbb{R}^{\mu p}$:
\begin{equation}
\label{eq:KKT}
\begin{split}
\nabla \MMD^2(\X, R) + J^\top\dual &= 0\\
\bar{h}(\X)                            &= 0.
\end{split}
\end{equation}
We re-formulate the KKT equations as a root-finding problem $\RMMD(\X,\dual) = 0$, where 
\begin{equation}
\renewcommand{\arraystretch}{1.5}
\label{eq:R_MMD=0}
\begin{split}
&\RMMD\colon \mathbb{R}^{\mu (n + p)} \to \mathbb{R}^{\mu (n + p)}, \quad \RMMD(\X,\dual) = 
\begin{pmatrix} 
\nabla \MMD^2(\X, R) + J^\top\dual \\
\bar{h}(\X) 
\end{pmatrix}.
\end{split}
\end{equation}
Deriving $\RMMD$ at $(\X,\vlambda)$ leads to
\begin{equation}
\renewcommand{\arraystretch}{1.5}
 \operatorname{D}\!\RMMD\left(\X,\dual\right) = 
 \begin{pmatrix} 
 \nabla^2 \MMD^2(\X, R) + S & J^\top\\ 
 J & 0
 \end{pmatrix}, 
 \quad S = \sum_{j=1}^{\mu p} \lambda_j \nabla^2 \bar{h}^j(\X),
\end{equation}
where $\operatorname{D}\!\RMMD(\X,\dual) \in \mathbb{R}^{\mu(n+p)\times \mu(n+p)}$ and $S\in \mathbb{R}^{\mu n \times \mu n}$.
We have now collected all to formulate the MMD-Newton algorithm: 
given search variables $\X_k\in\mathbb{R}^{\mu n}$ and $\dual_k \in\mathbb{R}^{\mu p}$ at iteration $k$, the next point $(\X_{k+1},\dual_{k+1})$ is obtained via solving 
\begin{equation}
\label{eq:Newton_step_S}
\operatorname{D}\!\RMMD\left(\X_k,\dual_k\right)
\begin{pmatrix} \X_{k+1} - \X_{k}\\ \dual_{k+1} - \dual_{k}\end{pmatrix} 
  = -\RMMD\left(\X_k,\dual_k\right). 
\end{equation}
The new iterate is well defined in case $\operatorname{D}\!\RMMD(\X_k,\dual_k)$ is invertible, e.g., if $J$ has full row rank and if $\nabla^2 \MMD(\X) + S$ is positive definite on the tangent space of the constraints~\cite{NocedalW99}. In the case where all the constraints $h_i$ are linear, the matrix $S$ vanishes, leading to a significant reduction in the computational
cost, as there is no need to compute the second-order derivative of the constraint.
If the problem is unconstrained, i.e., $p=0$ in~\cref{eq:MOPeq}, then the Newton method step is
\begin{equation} \label{eq:newton-step_unconstrained}
\X_{k+1} = \X_k - \left[\nabla^2 \MMD^2(\X_k)\right]^{-1}\nabla \MMD^2(\X_k). 
\end{equation}
We shall refer to \cref{eq:Newton_step_S} as the \emph{MMD-Newton method} (MMDN).

\section{MMD's derivatives and the properties} \label{sec:MMD-derivatives}
\subsection{Gradient and Hessian}
MMD's partial derivative w.r.t.~each objective point is: $\forall \ell\in[1..\mu]$,
\begin{equation} \label{eq:MMD-dy}
    \frac{\partial}{\partial \vy^{\,\ell}}\MMD^2(Y, R) = \frac{2}{\mu^2} \sum_{i\neq\ell} \frac{\partial k(\vy^{\,i}, \vy^{\,\ell})}{\partial \vy^{\, \ell}} - \frac{2}{\mu\lambda} \sum_{i=1}^{\lambda} \frac{\partial k(\vr^{\,i}, \vy^{\,\ell})}{\partial \vy^{\,\ell}}.
\end{equation}
Partial derivatives of the kernel $\partial k(\vy^{\,i}, \vy^{\,\ell})/\partial \vy^{\, \ell}$ can be calculated easily when the kernel function is specified. 
Now, consider the approximation set $Y$ is the image of $X\subseteq \R^n$ under $F$, i.e., $Y=F[X]$. The gradient of MMD w.r.t.~each decision point is:
\begin{equation} \label{eq:MMD-dx}
    \frac{\partial}{\partial \vx^{\,\ell}}\MMD^2(F[X], R) = \left(\frac{\partial}{\partial \vy^{\,\ell}}\MMD^2(Y, R)\right)\mathrm{D}F(\vx^{\,\ell}) \, ,
\end{equation}
where $\mathrm{D}F(\vx^{\,\ell}) \in \R^{k\times n}$ is the Jacobian of $F$ at $\vx^{\,\ell}$. Also, we choose $\partial\MMD(Y, R)/\partial \vy^{\,\ell}\in\R^{k}$ to be a row vector. Consider representing the set $X$ as a point $\X\in\R^{\mu n}$ (see~\cref{sec:MMD-Newton}), we define the MMD gradient as:
\begin{equation} \label{eq:MMD-gradient}
\nabla\MMD^2(\X, R)\coloneqq\left(\frac{\partial}{\partial \vx^{\,1}}\MMD^2(F[X], R),\ldots, \frac{\partial}{\partial \vx^{\,\mu}}\MMD^2(F[X], R)\right) \in \R^{\mu n}
\end{equation}
The second-order partial derivative w.r.t. the objective point is:
\begin{align} \label{eq:MMD-dy2}
\!\!&\frac{\partial^2}{\partial \vy^{\,m}\partial \vy^{\,\ell}}\MMD^2(Y, R)= \frac{2}{\mu^2} \frac{\partial^2}{\partial \vy^{\,m}\partial \vy^{\,\ell}} k(\vy^{\,m}, \vy^{\,\ell}), \; \text{ if } m\neq l \\
\!\!&\frac{\partial^2}{\partial \vy^{\,\ell}\partial \vy^{\,\ell}}\MMD^2(Y, R)= \frac{2}{\mu^2} \sum_{i\neq \ell} \frac{\partial^2}{\partial \vy^{\,\ell}\partial \vy^{\,\ell}} k(\vy^{\,i}, \vy^{\,\ell}) - \frac{2}{\mu\lambda} \sum_{i=1}^{\lambda} \frac{\partial^2}{\partial \vy^{\,\ell}\partial \vy^{\,\ell}}k(\vr^{\,i}, \vy^{\,\ell}), \; \text{ if } m=l
\end{align}
Similarly, we express the Hessian matrix of \MMD w.r.t. decision points:
\begin{align}
    &\frac{\partial}{\partial\vx^{\,m}\partial\vx^{\,\ell}}\MMD^2(F[X], R) \nonumber\\
    &= \mathrm{D}F(\vx^{\,m})^\top\left(\frac{\partial^2}{\partial \vy^{\,m}\partial \vy^{\,\ell}}\MMD^2(F[X], R)\right) \mathrm{D}F(\vx^{\,\ell}) +  \left(\frac{\partial}{\partial \vy^{\,\ell}}\MMD^2(F[X], R)\right) \mathrm{D}^2F(\vx^{\,\ell}) \delta^m_{\,\ell} \nonumber
\end{align}
where $\mathrm{D}^2F(\vx^{\,\ell}) \in\R^{k\times n \times n}$ is the second-order derivative (tensor of (1, 2)-type) of $F$ at $\vx^{\,\ell}$ and $\delta^m_{\,\ell} = 1$ iff. $m=\ell$. Similarly, we define the MMD Hessian as:
\begin{equation} \label{eq:MMD-Hessian}
\nabla^2\MMD^2(\X, R)\coloneqq\left(\frac{\partial}{\partial\vx^{\,m}\partial\vx^{\,\ell}}\MMD^2(F[X], R)\right)_{m,\ell\in[1..\mu]} \in \R^{\mu n \times \mu n}
\end{equation}

\subsection{Stationary condition}
Here, we investigate the first-order optimality condition of MMD for a special case $|Y| = |R|$. 
\begin{lemma}[First-order stationarity of MMD]\label{lemma:first-order-MMD} Assume $|Y| = |R|=\mu$ and a bounded stationary kernel $k$, e.g., the Gaussian kernel. The following condition holds:
$$
\frac{\partial}{\partial \vy^{\,\ell}}\MMD^2(Y, R) = 0 \Longleftrightarrow 
\sum_{\alpha\neq\ell} \sin (\langle \vomega, \vy^{\,\alpha} - \vy^{\,\ell} \rangle) = \sum_{\beta}\sin (\langle \vomega, \vr^{\,\beta} - \vy^{\,\ell} \rangle), \quad \forall \vomega\in\R^k.
$$
\end{lemma} 
\begin{proof} See~\cref{sec:proof-first-order-MMD} \end{proof}

\begin{remark}
    The condition of~\cref{lemma:first-order-MMD} holds in two cases: (1) $Y = R$ and (2) the terms in the sum cancel due to symmetry in the sets $Y$ and $R$. For instance, in $\R^2$, we have: $\vy^{\,\ell} = (0, 0), \vy^{\,1} = (a, 0), \vy^{\,2} = (-a, 0), \vr^{\,1} = (0, b), \vr^{\,2} = (0, -b)$ for some $a, b>0$. For an arbitrary $\vomega = (c\cos\theta, c\sin\theta)$ for some $c\geq 0, \theta\in[0,2\pi)$, we have $\langle \vomega, \vy^{\,\ell} - \vy^{\,1} \rangle = -ac \cos\theta, \langle \vomega, \vy^{\,\ell} - \vy^{\,2} \rangle = ac \cos\theta, \langle \vomega, \vy^{\,\ell} - \vr^{\,1} \rangle = -bc \sin\theta, \langle \vomega, \vy^{\,\ell} - \vr^{\,2} \rangle = bc \sin\theta$. Then, the condition holds for all $\vomega\in\R^2$.
\end{remark}
\begin{remark}
Taking Eq.~\eqref{eq:MMD-gradient-integral} (see \cref{sec:proof-first-order-MMD}) 
, we can bound the magnitude of the gradient: 
\begin{align*}
    \norm{\frac{\partial}{\partial \vy^{\,\ell}}\MMD^2(Y, R)}_2 
    & \leq \frac{2}{\mu}\left[\bar{D}(\vy^{\,\ell}, Y) + \bar{D}(\vy^{\,\ell}, R)\right]\mathbb{E}\left(\norm{\vomega}_2^2\right),
\end{align*}
where $\bar{D}(\vy^{\,\ell}, Y)=\frac{1}{\mu}\sum_{\alpha\neq\ell} \norm{\vy^{\,\alpha} - \vy^{\,\ell}}_2$ and $\bar{D}(\vy^{\,\ell}, R)=\frac{1}{\mu}\sum_{\beta} \norm{\vr^{\,\beta} - \vr^{\,\ell}}_2$ denote the average distance from $\vy^{\,\ell}$ to other points in $Y$ and to other points in $R$, respectively. The above upper bound is also controlled by $\mathbb{E}(\norm{\vomega}_2^2)=\int_{\R^k} \norm{\vomega}^2_2 G(\diff \vomega)$ which only depends on the choice of the kernel.
\end{remark}
\begin{example}
For a Gaussian kernel $k(\vy, \vy\,') = \exp(-\theta\norm{\vy - \vy\,'}_2^2)$, the spectral density of the corresponding measure $G$ is $S(\vomega) = \diff G / \diff \vomega = (4\pi\theta)^{-k/2}\exp(-\norm{\vomega}_2^2 / 4\theta)$, and we have:
$$
\int_{\R^k} \norm{\vomega}^2_2 G(\diff \vomega) = \int_{\R^k} \norm{\vomega}^2_2 S(\vomega) \diff \vomega = 2\theta k \implies \norm{\frac{\partial}{\partial \vy^{\,\ell}}\MMD^2(Y, R)}_2 \in \mathcal{O}\left(\frac{\theta k}{\mu}\right),
$$
suggesting that it might be sensible to set the parameter $\theta \in \mathcal{O}(\mu / k)$ to keep the gradient magnitude scaling only w.r.t.~the average distance from $\vy^{\,\ell}$ to other points in $Y$ and to $R$.
\end{example} 
\begin{theorem}[Necessary condition for Pareto optimality] \label{lemma:pareto-optimal-MMD} For bi-objective problems $F=(f_1, f_2)$, we assume $|R|=\mu$ and $|Y|=1$, the Gaussian kernel $k(\vy, \vy\,') = \exp(-\theta\norm{\vy - \vy\,'}_2^2)$ and its spectral density $\phi_k$, and $\vy=F(\vx)$. Suppose $\vx$ is locally efficient and the gradient of both objectives does not vanish, i.e., $\nabla f_1 \neq 0, \nabla f_2 \neq 0$, then the following implication holds:
$$
\frac{\partial}{\partial \vx}\MMD^2(\{F(\vx)\}, R) = 0 \implies \nabla f_1(\vx) + \lambda\nabla f_2(\vx) = 0, \lambda = \frac{\int_{\R^2} \nu_2 D(\vec{\nu};R, \vy, \theta) \phi_k(\vec{\nu}) \diff\vec{\nu}}{\int_{\R^2} \nu_1 D(\vec{\nu};R, \vy, \theta) \phi_k(\vec{\nu}) \diff\vec{\nu}},
$$
where $D(\vec{\nu};R, \vy, \theta)=\sum_{k=1}^{\mu}\sin (\sqrt{2\theta}\langle \vec{\nu}, \vr^{\,k} - \vy\, \rangle)$.
\end{theorem} 
\begin{proof} See~\cref{sec:proof-pareto-optimal-MMD}\end{proof}
\begin{example} \label{example:first-order-condition}
Note that we can well-approximate the sine function linearly when $\theta$ is small, i.e., $D(\vec{\nu};R,\vy,\theta)\to \sqrt{2\theta}\langle\vec{\nu}, \sum_k \vr^{\,k} - \vy\,\rangle$ as $\theta \to 0$. Denote by $\vec{m} = \mu^{-1}\sum_k \vr^{\,k}$ (the cluster center of the set $R$), we have $D(\vec{\nu};R,\vy,\theta)\approx \mu\sqrt{2\theta}\langle\vec{\nu}, \vec{m} - \vy\,\rangle$, which simplifies the integral further:
$$
\begin{array}{ll}
\int_{\R^2} \nu_1 D(\vec{\nu};R, \vy, \theta) \phi(\vec{\nu}) \diff\vec{\nu} \approx \mu\sqrt{2\theta}\int_{\R^2} \left(m_1\nu_1^2 + m_2\nu_1\nu_2\right)\phi(\vec{\nu}) \diff\vec{\nu} = m_1\mu\sqrt{2\theta} \\[0.5em]
\int_{\R^2} \nu_2 D(\vec{\nu};R, \vy, \theta) \phi(\vec{\nu}) \diff\vec{\nu} \approx \mu\sqrt{2\theta}\int_{\R^2} \left(m_1\nu_1\nu_2 + m_2\nu_2^2\right)\phi(\vec{\nu}) \diff\vec{\nu} = m_2\mu\sqrt{2\theta}
\end{array} \implies \lambda \approx \frac{m_2}{m_1}
$$
The result implies that when $\theta$ is small, the optimal point of MMD for a single point $\vx$ is an efficient point at which the normal space of the Pareto front has a slope approximately equal to $m_2/m_1$, meaning the normal space passes through the center of mass of the set $R$. We illustrate this observation in~\cref{fig:MMD-first-order-example} of \cref{sec:example-to-thm} on a hypothetical bi-objective problem. 
\end{example}

\subsection{Eigenspectrum of the Hessian}
Next, we analyze the eigenspectrum of each block of MMD's Hessian matrix. 
\begin{theorem}[Eigenspectrum of the Hessian block]\label{thm:eigenspectrum-Hessian-block}
Let $G$ denote the spectral measure corresponding to a bounded stationary kernel $k$ and $\vomega\sim G$. We assume $G$ has finite second and fourth moments, i.e., $C = \int_{\R^k} \vomega\vomega^\top G(\diff \vomega) < \infty, m_2 = \int_{\R^k} \norm{\vomega}^2_2 G(\diff \vomega) < \infty, m_4 = \int_{\R^k} \norm{\vomega}^4_2 G(\diff \vomega) < \infty$. Denote by $H^m_\ell = \frac{\partial^2}{\partial \vy^{\,m}\partial \vy^{\,\ell}}\MMD^2(Y, R)$ and let $\sigma(H)$ be an eigenvalue of $H$. We have the following bounds of $H$'s eigenvalues: 
\begin{align}
    &m\neq\ell,\, \frac{2}{\mu^2} \left(\sigma_{\text{min}}\left(C\right) - \frac{m_4}{2}\norm{\vy^{\,m} - \vy^{\,\ell}}_2^2\right) \leq \sigma(H^m_\ell) \leq \frac{2}{\mu^2}m_2, \label{eq:eigenspectrum-bound-off-diagonal}\\
    &m=\ell,\, \frac{2}{\mu}\left(\sigma_{\text{min}}\left(C\right) - m_2 - \frac{m_4}{2}\bar{D}^2(\vy^{\,\ell}, R)\right) \leq \sigma(H^\ell_\ell) \leq \frac{2}{\mu}\left(m_2 - \sigma_{\text{min}}\left(C\right) + \frac{m_4}{2}\bar{D}^2(\vy^{\,\ell}, Y)\right), 
    \label{eq:eigenspectrum-bound-diagonal}
\end{align}
where $\bar{D}^2(\vy^{\,\ell}, R)=\frac{1}{\mu}\sum_{\alpha} \norm{\vr^{\,\alpha} - \vy^{\,\ell}}_2^2$ and $\bar{D}^2(\vy^{\,\ell}, Y)=\frac{1}{\mu}\sum_{\beta\neq\ell} \norm{\vy^{\,\beta} - \vy^{\,\ell}}_2^2$.
\end{theorem}
\begin{proof} See~\cref{sec:proof-eigenspectrum-Hessian-block} \end{proof}
\begin{remark}
\cref{thm:eigenspectrum-Hessian-block} shows that the lower bound of the eigenvalue can be negative for the off-diagonal blocks ($m\neq\ell$) if points $\vy^{\,\ell}$ and $\vy^{\,m}$ are far away from each other, and for the diagonal blocks ($m=\ell$) if the point $\vy^{\,\ell}$ is far from the reference set $R$. In this case, some eigenvalues of the Hessian blocks can be negative. Particularly if the kernel $k$ induces an independent and isotropic measure $G$, then we have $C=\mathbb{E}(\vomega\vomega^\top)$ is multiple of the identity, and therefore $\sigma_{\text{min}}(C) = m_2 / k$, implying the lower bound $\frac{2}{\mu}\left(\sigma_{\text{min}}\left(C\right) - m_2 - \frac{m_4}{2}\bar{D}^2(\vy^{\,\ell}, R)\right)$ for $m=\ell$ is always negative.
\end{remark}
Finally, we bound the spectrum of the entire Hessian matrix.
\begin{corollary}[Eigenspectrum of the MMD Hessian]\label{corollary:eigenspectrum-Hessian} 
Taking the assumptions in~\cref{thm:eigenspectrum-Hessian-block}, we define the following value for each $\ell\in[1..\mu]$
\begin{equation} \label{eq:spectrum-radius}
    R_\ell = \frac{2}{\mu^2}\left[(2\mu - 1)m_2 - \sigma_{\text{min}}(C)\right] + \frac{m_4}{\mu}\left(\bar{D}^2(\vy^{\,\ell}, R) + \bar{D}^2(\vy^{\,\ell}, Y)\right).
\end{equation}
Then, the spectrum of the MMD Hessian is contained in the union of intervals $\left[-R_\ell, R_\ell\right]$, i.e., 
\begin{equation}
    \sigma\left(\left[\frac{\partial}{\partial\vy^{\,m}\partial\vy^{\,\ell}}\MMD^2(Y, R)\right]_{m,\ell\in[1..\mu]}\right) \in \bigcup_{\ell=1}^{\mu} \left[-R_\ell, R_\ell\right].
\end{equation}
\end{corollary}
\begin{proof}See~\cref{sec:proof-eigenspectrum-Hessian}\end{proof}
\begin{remark}
\cref{corollary:eigenspectrum-Hessian} implies two important observations: (1) there is no guarantee that the MMD Hessian is positive-semidefinite since the lower bound of its spectrum is negative. For the MMD-Newton method, it is necessary to scale up the spectrum to ensure the Newton direction is actually descending; (2) For a large approximation set $Y$, the spectrum is dominated by $\frac{m_4}{\mu}\left(\bar{D}^2(\vy^{\,\ell}, R) + \bar{D}^2(\vy^{\,\ell}, Y)\right)$, reflecting how far a point $\vy^{\,\ell}$ is from the reference set $R$ and the other approximation points in $Y$. 
\end{remark}

\section{Hybridization MMD-Newton with Evolutionary Algorithms}\label{sec:MMD-EA-hybrid}
Like other Newton-like algorithm, \algo has a super-linear convergence rate (even 
quadratic convergence under certain mild conditions). However, it might get stuck in local critical points and requires a reference set $R$. Consider the multi-objective evolutionary algorithms (MOEAs), which are gradient-free, population-based global search methods that have a slow convergence rate (sub-linear) \cite{liang:23}. It is natural to hybridize \algo with an MOEA (see~\cref{alg:MMDM-Newton}) to combine the advantages of both sides: we first execute the MOEA for several iterations to get close to the global Pareto front. Then, we warm-start  \algo with MOEA's Pareto approximation set (denoted by $Y_0$) as the initial points and a reference set $R$ created from $Y_0$, realizing a fast convergence to the Pareto front. 
\begin{algorithm}[h]
\caption{Hybridization of MMD-based Newton with multi-objective evolutionary algorithm \label{alg:MMDM-Newton}}
\textbf{Procedure:} MMDN-MOEA($F, h, g, \mu, \mathcal{A}, N_1, N_2$)\;
\textbf{Input:} objective function $F\colon \R^n \rightarrow \R^k$, equality constraint $h\colon \R^n \rightarrow \R^p$, inequality $g\colon \R^n \rightarrow \R^q$, approximation set size $\mu$, an MOEA $\mathcal{A}$, the number of iterations of MOEA $N_1$, the number of iterations of MMDN $N_2$\;
    $k\leftarrow 0, \; \varepsilon \leftarrow 10^{-6},\; \delta \leftarrow 0.08$\;
    \tcp{\small execute MOEA for $N_1$ iteration to get the initial approximate set $\X_0$ for MMDN}
    $\X_0\leftarrow \mathcal{A}(F, h, g, \mu, N_1)$\; 
    \tcp{\small generate reference set $R$.~see~\cref{sec:shifting-detail} for detail.}
    $Y_0 \leftarrow F(\X_0)$\;
    interpolate $Y_0$\;
    $R\leftarrow$ use $k$-means clustering to select $\mu$ points from $Y_0$\;
    Compute the unit normal vector $\vec{\eta}$ to the convex hull $\operatorname{Conv}(Y_0)$\;
    \lFor{$\vr \in R$}{$\vr \leftarrow \vr + \delta \vec{\eta}$\Comment*[f]{\small shift $R$ into the utopian region}} 
    $\vec{\lambda}_0 \leftarrow 0$\Comment*[r]{\small Lagrange multipliers} 
    \While{$\norm{\nabla\MMDM^2(\X_k)}_2 \leq \varepsilon \wedge k < N_2$}{
        precondition the MMD Hessian $\nabla^2 \MMDM^2(\X_k, R)$\Comment*[r]{\small see~\cref{sec:preconditioning}}
        apply backtracking line search to determine the step size $s$\;
        $\begin{pmatrix} \X_{k+1} - \X_{k}\\ \dual_{k+1} - \dual_{k}\end{pmatrix} = -s\left[\operatorname{D}\!\RMMD(\X_k,\dual_k)\right]^{-1}\RMMD(\X_k,\dual_k)$\Comment*[r]{\small Newton step~\cref{eq:Newton_step_S}}
        $Y_{k+1} \leftarrow F(\X_{k+1})$ \;
        $k \leftarrow k+1$ 
    }
    \textbf{Output:} $\X_k, Y_k$
\end{algorithm}

We generate the reference set $R$ from $Y_0$ by considering $Y_0$ as a rough guess of the Pareto front. However, we face two issues: (1) $Y_0$ might not be well-shaped, e.g., a non-uniform density in space; (2) since we do not know the Pareto front, we should shift the set $Y_0$ into the utopian region of the objective space (containing points that dominate the Pareto front). To solve these issues, we first perform a filling of $Y_0$ and then use $k$-means clustering to select $\mu$ points from it, which ensures that $Y_0$ has an even density. Second, to determine the shifting direction, we compute the unit normal direction to the convex hull $\operatorname{Conv}(Y_0)$, denoted by $\vec{\eta}\in\R^k$. Then, $R$ is created by shifting each point in $Y_0$ with $\delta\vec{\eta}$, where $\delta >0$ is a scaling constant. We include details of the reference set generation in~\cref{sec:shifting-detail}. As proven in~\cref{corollary:eigenspectrum-Hessian}, the MMD's Hessian is not necessarily positive-definite. To ensure the Newton update in \cref{eq:Newton_step_S} is a descending direction, we must precondition the Hessian matrix. To realize this, we utilize the preconditioning Algorithm 3.3 from~\cite{NocedalW99} (details in~\cref{sec:preconditioning}). Lastly, we employ a backtracking line search with Armijo's condition~\cite{NocedalW99} to determine the step-size $s$ of the Newton step.

\section{Experiments} \label{sec:experiments}
\paragraph{Experimental setup} We test the performance of \cref{alg:MMDM-Newton} on 11 widely used benchmark problems for numerical multi-objective optimization: (1) ZDT1-4 problems~\cite{zitzler2000comparison} represent the major difficulties in solving bi-objective problems, e.g., convex (ZDT1) and concave Pareto front (ZDT2), disconnectedness of Pareto front (ZDT3), local Pareto front (ZDT4). We did not use ZDT5 and ZDT6 since the former is a discrete problem, and for the latter, its gradient is not Lipschitz continuous at the efficient set. Hence, our algorithm is not applicable to it. (2) DTLZ1-7 problems~\cite{Deb2002DTLZ} cover some challenges in three-objective scenarios, e.g., exponentially many local Pareto fronts (DTLZ3), hardness to cover the front uniformly (DTLZ4), degenerated Pareto front (DTLZ5 and 6). We chose three state-of-the-art MOEAs to warm-start \algo in~\cref{alg:MMDM-Newton}: NSGA-II~\cite{DebAPM02}, NSGA-III~\cite{DebJ14,JainD14}, MOEA/D~\cite{ZhangL07}, which are widely applied in real-world applications. All algorithms are executed for 30 independent runs on each problem.\footnote{The source code: \url{https://anonymous.4open.science/r/HypervolumeDerivatives-5287}}

\paragraph{Fair comparison of \algo to MOEA} For the hybrid algorithm, we execute the MOEA for 300 iterations ($N_1=300$) and then run \algo for five iterations ($N_2=5$). We used automatic differentiation (AD)~\cite{0002126} techniques to compute the Jacobian and Hessian of the objective functions. To show the effectiveness of \algo, we compare the hybrid with MOEA alone on each problem. For fairness, we need to calculate the number of function evaluations equivalent to five iterations of \algo since the MOEA does not use the gradient/Hessian of the objective function. Theoretically, for any differentiable function $F\colon\R^n\rightarrow\R^k$, the time complexity to compute the Jacobian is bounded by $4k\operatorname{OPS}(F)$ in the reverse-mode AD~\cite{Margossian19}, where $\operatorname{OPS}(F)$ is the number of addition/multiplication operations in one function evaluation. To make it realistic, we empirically measured the CPU time of AD over all problems, yielding an average value of ca.~$1.47\operatorname{OPS}(F)$. For the Hessian computation, the empirical estimation gives an average value of ca.~$1.89\operatorname{OPS}(F)$. The details are included in~\cref{sec:AD-CPU-time}. In the five iterations of \algo, we record the actual number of Jacobian and Hessian calls (which vary from one run to another due to the backtracking line search) to compute the equivalent number of function evaluations for MOEA.

\paragraph{Hyperparameters} The \algo method depends on the choice of the kernel function $k$ and its length-scale parameter $\theta$. Here, we propose a simple heuristic to determine the best pair $(k, \theta)$ for each test problem and baseline MOEA: we simply choose the pair that gives the smallest condition number of the Hessian $\nabla^2\MMD^2(Y,R)$ after preconditioning. Details are included in~\cref{sec:hyperparameter-setting}. For all three MOEAs, we took their default hyperparameter setting as in the \texttt{Pymoo} library~\cite{BlankD20} (Apache-2.0 License; version $0.6.1.1$). For the reference set generation step in~\cref{alg:MMDM-Newton}, we took the hyperparameter setting in~\cite{WangRUCS24}.
\begin{table}[ht]
\setlength{\tabcolsep}{2.7pt}
\fontsize{7.5pt}{10.5pt}\selectfont
\caption{Performance of the \texttt{\algo+ MOEA} hybrid (\cref{alg:MMDM-Newton}) compared to the same \texttt{MOEA alone} under the same function evaluation budget. 
We measure the average Hausdorff distance $\Delta_2$ between the Pareto front and the final approximation set produced by each algorithm. We show the median and 10\% - 90\% quantile range (in the brackets) of $\Delta_2$ values over 30 independent runs. Mann–Whitney U test (with 5\% significance level) is used to check the significance, where the Holm-Sidak method is used to adjust the $p$-value for multiple testing. Three symbols $\uparrow$, $\leftrightarrow$, and $\downarrow$ indicate if the hybrid is significantly better, no difference, or worse than the baseline MOEA, respectively.
\label{tab:MMD-results}}
\centering
\vspace{2mm}
\begin{minipage}{0.48\textwidth}
\begin{tabular}{llll}
\toprule
Baseline & Problem & \algo + MOEA & MOEA alone  \\
\midrule
NSGA-II & ZDT1 & 0.0046(5.96e-04)$\uparrow$ & 0.0057(7.47e-04) \\
NSGA-II & ZDT2 & 0.0048(2.29e-04)$\uparrow$ & 0.0059(6.46e-04) \\
NSGA-II & ZDT3 & 0.0063(1.60e-03)$\uparrow$ & 0.0068(8.10e-04) \\
NSGA-II & ZDT4 & 0.0050(4.92e-04)$\uparrow$ & 0.0054(5.04e-04) \\
NSGA-II & DTLZ1 & 0.0139(6.98e-04)$\uparrow$ & 0.0179(1.58e-03) \\
NSGA-II & DTLZ2 & 0.0362(1.62e-03)$\uparrow$ & 0.0460(3.70e-03) \\
NSGA-II & DTLZ3 & 0.0403(1.14e-02)$\uparrow$ & 0.0473(3.36e-02) \\
NSGA-II & DTLZ4 & 0.0357(1.82e-03)$\uparrow$ & 0.0448(3.13e-03) \\
NSGA-II & DTLZ5 & 0.0045(6.39e-05)$\uparrow$ & 0.0046(1.30e-04) \\
NSGA-II & DTLZ6 & 0.0587(4.92e-02)$\leftrightarrow$ & 0.0573(4.23e-02) \\
NSGA-II & DTLZ7 & 0.0376(3.77e-03)$\uparrow$ & 0.0484(7.54e-03) \\
NSGA-III & ZDT1 & 0.0048(3.30e-06)$\uparrow$ & 0.0052(3.83e-05) \\
NSGA-III & ZDT2 & 0.0044(5.82e-06)$\uparrow$ & 0.0045(1.99e-05) \\
NSGA-III & ZDT3 & 0.0070(2.26e-03)$\uparrow$ & 0.0083(1.47e-04) \\
NSGA-III & ZDT4 & 0.0079(3.01e-02)$\leftrightarrow$ & 0.0080(5.61e-02) \\
NSGA-III & DTLZ1 & 0.0113(1.49e-04)$\uparrow$ & 0.0132(3.29e-05) \\
\end{tabular}
\end{minipage}
\hfill
\begin{minipage}{0.48\textwidth}
\begin{tabular}{llll}
NSGA-III & DTLZ2 & 0.0302(4.80e-05)$\uparrow$ & 0.0346(2.13e-05) \\
NSGA-III & DTLZ3 & 0.0310(1.43e-03)$\uparrow$ & 0.0350(1.49e-03) \\
NSGA-III & DTLZ4 & 0.0302(4.89e-05)$\uparrow$ & 0.0346(3.59e-05) \\
NSGA-III & DTLZ5 & 0.0046(3.12e-04)$\uparrow$ & 0.0250(6.56e-04) \\
NSGA-III & DTLZ6 & 0.0521(3.10e-02)$\uparrow$ & 0.0869(1.32e-01) \\
NSGA-III & DTLZ7 & 0.0355(1.84e-03)$\uparrow$ & 0.0588(1.10e-03) \\
MOEA/D & ZDT1 & 0.0067(1.77e-02)$\uparrow$ & 0.0158(7.97e-02) \\
MOEA/D & ZDT2 & 0.0078(7.23e-03)$\uparrow$ & 0.0589(2.72e-01) \\
MOEA/D & ZDT3 & 0.0703(1.15e-01)$\leftrightarrow$ & 0.0432(1.42e-01) \\
MOEA/D & ZDT4 & 0.0097(3.26e-02)$\uparrow$ & 2.0001(4.18e+00) \\
MOEA/D & DTLZ1 & 0.0117(3.39e-04)$\uparrow$ & 0.0132(3.12e-04) \\
MOEA/D & DTLZ2 & 0.0302(2.33e-05)$\uparrow$ & 0.0346(2.12e-05) \\
MOEA/D & DTLZ3 & 0.0336(4.03e-03)$\uparrow$ & 0.0406(1.65e-02) \\
MOEA/D & DTLZ4 & 0.0301(1.54e-04)$\uparrow$ & 0.0346(6.40e-01) \\
MOEA/D & DTLZ5 & 0.0226(3.33e-02)$\leftrightarrow$ & 0.0222(4.23e-05) \\
MOEA/D & DTLZ6 & 0.0592(6.13e-02)$\uparrow$ & 0.0883(1.62e-01) \\
MOEA/D & DTLZ7 & 0.1045(1.64e-02)$\downarrow$ & 0.0990(6.50e-03) \\ \midrule
win/tie/loss &  & 28/4/1 &  \\
\bottomrule
\end{tabular}
\end{minipage}
\end{table}


\paragraph{Results} We show the results in~\cref{tab:MMD-results}. For the performance assessment, we took the well-known average Hausdorff metric $\Delta_p$ ($p=2$), which measures the 
averaged Hausdorff distance between the Pareto front and the final approximation set given by~\cref{alg:MMDM-Newton}. 
In the table, column \texttt{Baseline} indicates the MOEA used to warm-start the \algo method, \texttt{\algo + MOEA} denotes the hybrid algorithm, \texttt{MOEA alone} means the same baseline MOEA executed under an evaluation budget equivalent to the hybrid. We observed that the hybrid \texttt{\algo + MOEA} outperforms \texttt{MOEA alone} for 28 cases out of 33, verifying \algo can refine the Pareto approximation set much more efficiently compared to MOEAs. For the one case we lose, i.e., MOEA/D algorithm as the baseline on DTLZ7 problem, we inspect the optimization data and find out that the final population of MOEA/D after 300 iterations concentrates on some sub-region of the Pareto front, hence the reference set and the starting points of \algo are of very low quality. While \algo has a certain capability to recover the Pareto front, it loses on this example against the 
global solver.

\section{Limitations}  \label{sec:limitations}
The first-order stationary condition in~\cref{lemma:pareto-optimal-MMD} is derived for a special case: bi-objective problems with a single approximation point. The \algo method, as a Newton method, is a local search method and can get stuck in local optima. When hybridizing \algo with MOEAs (global search), the performance of \algo can be affected by the quality (e.g., the diversity) of the Pareto approximation set of MOEA, as shown by the DTLZ7 problem with MOEA/D case in the previous section.

\section{Conclusion} \label{sec:conclusion}
This paper proposes using MMD to measure the distance between a Pareto approximation set and a reference set for solving numerical multi-objective optimization problems. We designed the MMD-Newton method (MMDN) to minimize the distance between the approximation and the reference sets, which is facilitated by analytical expressions of MMD's gradient and Hessian matrix. Also, we analyzed the first-order stationary condition of MMD and provided a bound for the eigenspectrum of its Hessian, implying that the Hessian is often indefinite. Hence, we suggest preconditioning the Hessian in the Newton method.\\
We observe an important property of MMD: while minimizing the distance, it also keeps the diversity of the approximation set. Consequently, we found out that the \algo method can improve the coverage of the Pareto front (attributed to the diversity maintenance) even if the reference set has poor coverage. Therefore, we further propose hybridizing \algo with multi-objective evolutionary algorithms (MOEAs) such that \algo can help MOEAs to improve the coverage and local convergence rate while MOEAs can provide reasonable starting points for \algo. The resulting hybrid algorithm significantly outperforms the MOEAs alone on 11 widely used benchmark problems, supporting the usefulness of \algo.\\
For further work, we plan to investigate the impact of the choice of kernel functions and their length-scale on the performance of \algo. Despite that, we propose a simple heuristic to select the kernel and length-scale (see~\cref{sec:hyperparameter-setting}), the theoretical implication of the kernel is unclear except for the impact on the eigenspectrum of the Hessian, as stated in~\cref{thm:eigenspectrum-Hessian-block}.

\bibliographystyle{plain}
\bibliography{ref}


\appendix

\section{Proofs}
\subsection{Proof of \cref{lemma:first-order-MMD}} \label{sec:proof-first-order-MMD}
\begin{proof}
We take the spectral representation of MMD (\cref{eq:spectrum-MMD}): 
\begin{align*}
    &\MMD^2(Y, R) \\
    &= \int_{\R^k} \left\vert \widehat{\varphi}_Y  - \widehat{\varphi}_R \right\vert^2 G\left(\diff \vec{\omega}\right) \\
    &=\frac{1}{\mu^2}\int_{\R^k} \left[\sum_{\alpha,\beta} e^{\iu \langle\vomega, \vy^{\,\alpha} - \vy^{\,\beta}\rangle} - \sum_{\alpha, \beta} \left(e^{\iu \langle\vomega, \vy^{\,\alpha} - \vr^{\,\beta}\rangle} + e^{\iu \langle\vomega, \vr^{\,\alpha} - \vy^{\,\beta}\rangle}\right) + \sum_{\alpha, \beta} e^{\iu \langle\vomega, \vr^{\,\alpha} - \vr^{\,\beta}\rangle} \right]G\left(\diff \vomega\right)
\end{align*} 
Notice that:
\begin{align}
&\frac{\partial}{\partial \vy^{\,\ell}}\sum_{\alpha,\beta} e^{\iu \langle \vomega, \vy^{\,\alpha} - \vy^{\,\beta} \rangle}
     = \iu \vomega\sum_{\alpha} e^{\iu \langle \vomega, \vy^{\,\ell} - \vy^{\,\alpha} \rangle} - e^{-\iu \langle \vomega, \vy^{\,\ell} - \vy^{\,\alpha} \rangle} = 2\vomega \sum_{\alpha\neq\ell}\sin\left(\langle \vomega, \vy^{\,\alpha} - \vy^{\,\ell} \rangle\right) \\
&\frac{\partial}{\partial \vy^{\,\ell}}\sum_{\alpha,\beta} \left(e^{\iu \langle \vomega, \vy^{\,\alpha} - \vr^{\,\beta} \rangle} + e^{\iu \langle \vomega, \vr^{\,\alpha} - \vy^{\,\beta} \rangle}\right) = 
2\vomega \sum_{\alpha}\sin\left(\langle \vomega, \vr^{\,\alpha} - \vy^{\,\ell} \rangle\right)
\end{align}
We express the MMD first-order derivative as follows: 
\begin{align}
    &\frac{\partial}{\partial \vy^{\,\ell}}\MMD^2(Y, R) 
    =\frac{2}{\mu^2}\int_{\R^k}\vomega\bigg(\sum_{\alpha\neq\ell} \sin (\langle \vomega, \vy^{\,\alpha} - \vy^{\,\ell} \rangle) - \sum_{\beta}\sin (\langle \vomega, \vr^{\,\beta} - \vy^{\,\ell} \rangle)\bigg) G\left(\diff \vomega\right) \label{eq:MMD-gradient-integral}
\end{align}
For critical points of MMD, i.e., $\partial/\partial \vy^{\,\ell}\MMD^2(Y, R) =0$, the following condition must hold since $G$ is non-degenerating, non-negative Borel measure:
\begin{align*}
    &\sum_{\alpha\neq\ell} \sin (\langle \vomega, \vy^{\,\alpha} - \vy^{\,\ell} \rangle) = \sum_{\beta}\sin (\langle \vomega, \vr^{\,\beta} - \vy^{\,\ell} \rangle), \quad\forall \vomega
\end{align*}
\end{proof}

\subsection{Proof of \cref{lemma:pareto-optimal-MMD}} \label{sec:proof-pareto-optimal-MMD}
\begin{proof}
    Assume $\vec{x}$ is locally efficient. We must have $\nabla f_1(\vec{x}) = -c\nabla f_2(\vec{x})$ for some $c>0$. Denote by $\vec{a} = \nabla f_2(\vec{x})$, $\nabla f_1(\vec{x}) = -c\vec{a}$, and $\vec{n} = (-c, 1)^\top$, we have $\operatorname{D}\!F(\vx)=\vec{n}\vec{a}^\top$. Note that $-1/c$ is the slope of the tangent space/line of the Pareto front at $\vy$ ($c$ is the slope of the normal space). The question for critical points is equivalent to asking at what $c$ value, the MMD gradient is zero. We consider only one approximation point, i.e., $Y = \{\vy\}, \vy = F(\vx)$. The critical point condition is:
\begin{align*}
\frac{\partial}{\partial \vx}\MMD^2(F(\vx), R) &= \left(\frac{\partial}{\partial \vy}\MMD^2(\{\vy\}, R)\right)\operatorname{D}\!F(\vx) \\ 
&= -\frac{2\vec{a}^\top}{\mu}\int_{\R^2} \langle\vomega, \vec{n}\rangle\sum_{k}\sin (\langle \vomega, \vr^{\,k} - \vy\, \rangle)\phi_k(\vomega) \diff\vomega = 0,
\end{align*}
where for the Gaussian kernel, its spectral density is $\phi_k(\vomega) = (4\pi\theta)^{-k/2}\exp(-\norm{\vomega}_2^2 / 4\theta)$. Let $\vec{\nu} = \vomega / \sqrt{2\theta}$, the above integral can be simplified to: 
\begin{align*}
&\int_{\R^2} \langle\vec{\nu}, \vec{n}\rangle\underbrace{\sum_{k}\sin (\sqrt{2\theta}\langle \vec{\nu}, \vr^{\,k} - \vy\, \rangle)}_{D(\vec{\nu};R, \vy, \theta)} \phi_k(\vec{\nu}) \diff\vec{\nu} = 0 \\
&\Longleftrightarrow c\int_{\R^2} \nu_1 D(\vec{\nu};R, \vy, \theta) \phi(\vec{\nu}) \diff\vec{\nu} = \int_{\R^2} \nu_2 D(\vec{\nu};R, \vy, \theta) \phi_k(\vec{\nu}) \diff\vec{\nu} \\
&\Longleftrightarrow c = \frac{\int_{\R^2} \nu_2 D(\vec{\nu};R, \vy, \theta) \phi_k(\vec{\nu}) \diff\vec{\nu}}{\int_{\R^2} \nu_1 D(\vec{\nu};R, \vy, \theta) \phi_k(\vec{\nu}) \diff\vec{\nu}}
\end{align*}
where $\phi_k(\vec{\nu}) = (2\pi)^{-1}\exp(-\norm{\vec{\nu}}^2_2/2)$ is the density function of the standard bi-variate Gaussian.
\end{proof} 

\subsection{Proof of \cref{thm:eigenspectrum-Hessian-block}} \label{sec:proof-eigenspectrum-Hessian-block}
\begin{proof}
Firstly, we consider the spectral representation of the Hessian by differentiating~\cref{eq:MMD-gradient-integral}: 
\begin{align*}
&\frac{\partial^2}{\partial \vy^{\,m}\partial \vy^{\,\ell}}\MMD^2(Y, R) \\
&=
\begin{dcases}
\frac{2}{\mu^2} \int_{\R^k} \vomega \vomega^\top \left( \sum_{\alpha} \cos (\langle \omega, \vr^{\,\alpha} - \vy^{\,\ell} \rangle) - \sum_{\beta\neq \ell} \cos (\langle \vomega, \vy^{\,\beta} - \vy^{\,\ell} \rangle) \right) G(\diff \vomega), & m = \ell, \\
\frac{2}{\mu^2} \int_{\R^k} \vomega \vomega^\top \cos (\langle \vomega, \vy^{\,m} - \vy^{\,\ell} \rangle) G(\diff \vomega), & m \neq \ell.
\end{dcases}
\end{align*}
The Hessian matrix is of the form $\mathbb{E} [\vomega \vomega^\top]$, which is a covariance-like structure weighted by cosine terms. Such matrices are symmetric and positive semi-definite when the weight is nonnegative. Since cosines oscillate between $[-1,1]$, they can lead to sign changes in eigenvalues.\\
Case $m\neq\ell$: Consider the Rayleigh quotient: $\forall \vz\in\R^k, \vec{z}^\top (H^m_\ell) \vec{z}/\vz^\top \vz$, surjective in $[\sigma_{\text{min}}(H^m_\ell), \sigma_{\text{max}}(H^m_\ell)]$.
An upper bound is trivial:
\begin{align} \label{eq:Hessian-off-diagonal-upper-bound}
\sigma_{\text{max}}(H^m_\ell) &= \sup_{\vz\in \R^k}\frac{\vec{z}^\top H^m_\ell \vec{z}}{\vz^\top \vz} \nonumber \\
&= \frac{2}{\mu^2}  \sup_{\vz\in \R^k} \int_{\R^k} \frac{(\vz^\top\vomega)^2}{\norm{\vz}_2^2} \cos (\langle \vomega, \vy^{\,m} - \vy^{\,\ell} \rangle) G(\diff \vomega)\leq \frac{2}{\mu^2}\underbrace{\mathbb{E}\left[\norm{\vomega}_2^2\right]}_{m_2}
\end{align}
Taking the inequality $1 - \cos x\leq x^2 / 2$, we have a lower bound:
\begin{align}
    \sigma_{\text{min}}(H^m_\ell) &= \inf_{\vz\in \R^k}\frac{\vec{z}^\top H \vec{z}}{\vz^\top \vz} =
    \frac{2}{\mu^2}  \inf_{\vz\in \R^k}\int_{\R^k} \frac{(\vz^\top\vomega)^2}{\norm{\vz}_2^2}\cos (\langle \vomega, \vy^{\,m} - \vy^{\,\ell} \rangle) G(\diff \vomega)  \nonumber \\
    &\geq \frac{2}{\mu^2}  \inf_{\vz\in \R^k}\int_{\R^k}\frac{(\vz^\top\vomega)^2}{\norm{\vz}_2^2} \left(1  - \frac{1}{2}\langle \vomega, \vy^{\,m} - \vy^{\,\ell} \rangle^2\right) G(\diff \vomega) \nonumber\\
    & \geq \frac{2}{\mu^2}\bigg[\sigma_{\text{min}}\underbrace{\left(\mathbb{E}\left(\vomega\vomega^\top\right)\right)}_{C} - \frac{\norm{\vy^{\,m} - \vy^{\,\ell}}_2^2}{2}\underbrace{\mathbb{E}\left(\norm{\vomega}_2^4 \right)}_{m_4} \bigg] \label{eq:Hessian-off-diagonal-lower-bound}
\end{align}
Case $m=\ell$: 
\begin{align*}
    H^m_\ell = \underbrace{\frac{2}{\mu^2} \sum_{\alpha} \int_{\R^k} \vomega \vomega^\top \cos(\langle \omega, \vr^{\,\alpha} - \vy^{\,\ell} \rangle) G(\diff \vomega)}_{H_1} - \underbrace{\frac{2}{\mu^2} \sum_{\beta\neq \ell}  \int_{\R^k} \vomega \vomega^\top  \cos (\langle \vomega, \vy^{\,\beta} - \vy^{\,\ell} \rangle) G(\diff \vomega)}_{H_2}
\end{align*}
Applying \cref{eq:Hessian-off-diagonal-upper-bound,eq:Hessian-off-diagonal-lower-bound}, we have:
\begin{align*}
     &\frac{2}{\mu}\left(\sigma_{\text{min}}\left(C\right) - \frac{m_4}{2}\bar{D}^2(\vy^{\,\ell}, R)\right) \leq  \sigma_{\text{min}}(H_1) \leq  \sigma_{\text{max}}(H_1) \leq \frac{2}{\mu} m_2 \\
     &\frac{2}{\mu}\left(\sigma_{\text{min}}\left(C\right) - \frac{m_4}{2}\bar{D}^2(\vy^{\,\ell}, Y)\right) \leq  \sigma_{\text{min}}(H_2) \leq  \sigma_{\text{max}}(H_2) \leq \frac{2}{\mu} m_2,
\end{align*}
where $\bar{D}^2(\vy^{\,\ell}, R)=\frac{1}{\mu}\sum_{\alpha} \norm{\vr^{\,\alpha} - \vy^{\,\ell}}_2^2$ and $\bar{D}^2(\vy^{\,\ell}, Y)=\frac{1}{\mu}\sum_{\beta\neq\ell} \norm{\vy^{\,\beta} - \vy^{\,\ell}}_2^2$.
Applying Weyl's inequality to $H^m_\ell$, i.e., $\sigma_{\text{min}}(H^m_\ell)\geq \sigma_{\text{min}}(H_1) - \sigma_{\text{max}}(H_2)$ and $\sigma_{\text{max}}(H)\geq \sigma_{\text{max}}(H_1) - \sigma_{\text{min}}(H_2)$ leads to:
\begin{align}
    \frac{2}{\mu}\left(\sigma_{\text{min}}\left(C\right) - m_2 - \frac{m_4}{2}\bar{D}^2(\vy^{\,\ell}, R)\right) \leq \sigma(H^m_\ell) \leq \frac{2}{\mu}\left(m_2 - \sigma_{\text{min}}\left(C\right) + \frac{m_4}{2}\bar{D}^2(\vy^{\,\ell}, Y)\right).
\end{align}
\end{proof}

\subsection{Proof of \cref{corollary:eigenspectrum-Hessian}} \label{sec:proof-eigenspectrum-Hessian}
\begin{proof}
    We use Gershgorin's circle theorem for block matrices~\cite{tretter2008spectral}. Consider a symmetric matrix $H = (H^m_\ell) \in \R^{\mu k \times \mu k}, H^m_\ell\in \R^{k \times k}$. Let $\sigma(H), \sigma(H)$ represent the singular values and eigenvalues of $H$, respectively. If we denote
    $$
    G_\ell = \sigma(H^\ell_\ell) \cup \left\{\bigcup_{i=1}^k \left[ \sigma_i(H^\ell_\ell) - S_\ell, \sigma_i(H^\ell_\ell) + S_\ell\right]\right\}, \quad S_\ell = \sum_{m\neq\ell} \sigma_{\text{max}}(H^m_\ell).
    $$
    Then, the eigenvalues of $H$ are contained in the union of $G_\ell$. We apply this theorem to our case. Taking~\cref{eq:eigenspectrum-bound-off-diagonal}, we have: 
    \begin{align*}
        &\sigma_{\text{max}}(H^m_\ell) \leq \frac{2}{\mu^2}\max\left\{\left\lvert\sigma_{\text{min}}(C) - \frac{m_4}{2}\norm{\vy^{\,m} - \vy^{\,\ell}}_2^2\right\rvert, m_2\right\}\leq  \frac{2}{\mu^2}\left(m_2 + \sigma_{\text{min}}(C) + \frac{m_4}{2}\norm{\vy^{\,m} - \vy^{\,\ell}}_2^2\right)\\
        &S_\ell \leq \frac{2(\mu - 1)}{\mu^2}\left(m_2 + \sigma_{\text{min}}(C)\right) + \frac{m_4}{\mu^2}\sum_{m\neq\ell}\norm{\vy^{\,m} - \vy^{\,\ell}}_2^2.
    \end{align*}
    Denote by $-B_\ell$ and $B_\ell$ the lower and upper bounds in~\cref{eq:eigenspectrum-bound-diagonal}, i.e., $\sigma(H^\ell_\ell) \in [-B_\ell, B_\ell]$. We have: 
    \begin{align*}
        \bigcup_{i=1}^k \left[ \sigma_i(H^\ell_\ell) - S_\ell, \sigma_i(H^\ell_\ell) + S_\ell\right] &\subseteq  \left[ \sigma_{\text{min}}(H^\ell_\ell) - S_\ell, \sigma_{\text{max}}(H^\ell_\ell) + S_\ell\right] \\ 
        &\subseteq [-B_\ell-S_\ell, B_\ell + S_\ell] \subseteq [-R_\ell, R_\ell],
    \end{align*}
    and $\sigma(H^\ell_\ell) \in [0, B_\ell] \subseteq [-R_\ell, R_\ell]$. Hence we have $G_\ell\subseteq [-R_\ell, R_\ell]$, which completes the proof.
\end{proof}

\subsection{Example to illustrate \cref{lemma:pareto-optimal-MMD}} \label{sec:example-to-thm}
Here, we provide a simple, but visual example to illustrate \cref{lemma:pareto-optimal-MMD} and~\cref{example:first-order-condition}. We consider a very simple bi-objective problem: $F=(f_1, f_2), f_1 = (x_1 - 1)^2 + (x_2 - 1)^2, f_2 = (x_1 + 1)^2 + (x_2 + 1)^2, x_1, x_2\in[-2, 2]$, whose Pareto front is $f_2 = 2\left(2 - \sqrt{f_1/2}\right)^2, f_1\in[0,8]$ (indicated by the pink curve in~\cref{fig:MMD-first-order-example}). For an efficient point $\vec{x}$, it must satisfy $\exists \lambda, \nabla f_1(\vec{x}) + \lambda \nabla f_2(\vec{x}) = 0$ by the KKT theorem. The normal space of the Pareto front at $\vec{y} = F(\vec{x})$ is identified with the slope $f_2 / f_1 = \lambda$. We start the \algo method with the Gaussian kernel $k(\vec{x}, \vec{y})=\exp(-\theta\lVert\vec{x} - \vec{y}\rVert^2_2)$ from a single approximation point, which is already on the Pareto front (black cross), the method converges to the stationary point (black star), at which the normal space (black dashed line) approximately pass through the center of mass of four reference points (black points) when the length-scale $\theta$ is small (left sub-figure) since by~\cref{example:first-order-condition}, $\lambda\approx m_2 /m_1$, where $(m_1, m_2)$ is the center of reference points minus $\vec{y}$. When $\theta$ gets larger (right sub-figure), the normal space and the center of reference points differ substantially, as discussed in~\cref{example:first-order-condition}.
\begin{figure}[t]
  \centering
  \includegraphics[width=.48\textwidth, trim=130mm 23mm 0mm 22mm, clip]{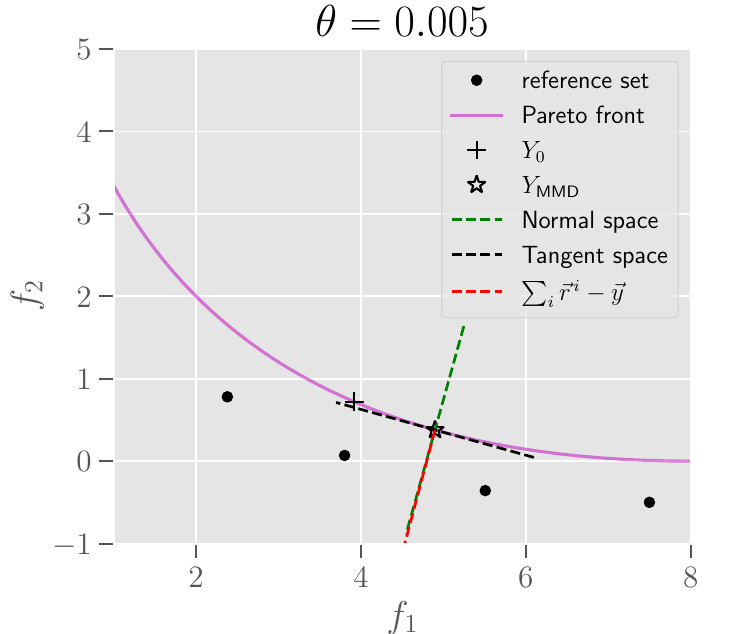}
  \hfill
  \includegraphics[width=.48\textwidth, trim=130mm 23mm 0mm 22mm, clip]{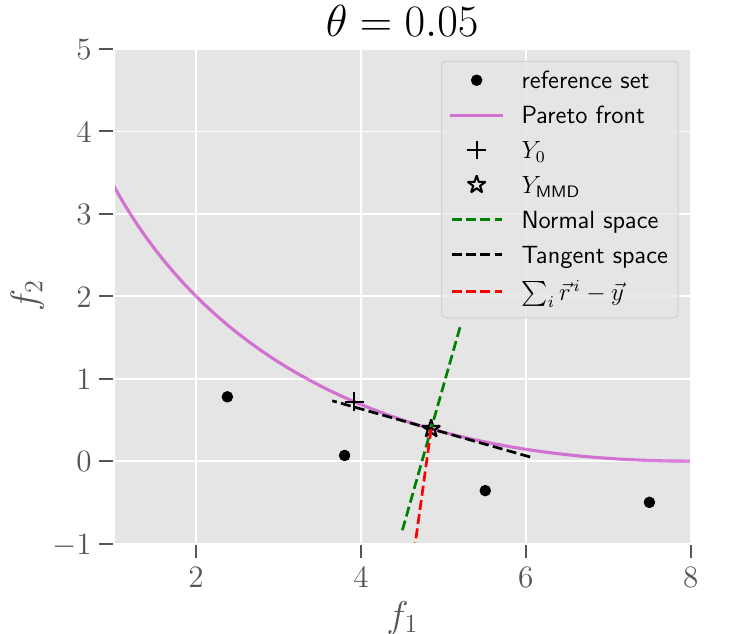}
  \caption{Example of the Pareto optimal condition of MMD with the Gaussian kernel for $|Y| = 1, |R|=4$ on a simple bi-objective problem: $f_1 = (x_1 - 1)^2 + (x_2 - 1)^2, f_2 = (x_1 + 1)^2 + (x_2 + 1)^2$. When the length-scale $\theta$ is small (left figure), at the optimal point (star marker), the normal space (green line) approximately passes through the center of mass of the reference set $R$. When $\theta$ is larger (right figure), the normal space and the center of mass of $R$ differ substantially at the optimal point.
  \label{fig:MMD-first-order-example}
  }
\end{figure}

\section{Details of reference set generation in~\cref{alg:MMDM-Newton}} \label{sec:shifting-detail}

To generate a reference set $R$ from the final population of a MOEA, denoted $Y_0$, we follow a procedure similar to that in~\cite{math13101626}, using the same parameter values as in~\cite{WangRUCS24}. This aims to produce a more uniformly distributed reference set for the Newton step. The process of deriving $R$ from $Y_0$ involves three main steps: a) filling, b) reduction, and c) shifting.

For a), \textbf{filling}, we first perform component detection using DBSCAN clustering~\cite{EsterKSX96}. Each detected component is then filled using a method that depends on the number of objectives. For bi-objective problems $F=(f_1, f_2)$, we sort the solutions by the first objective value to obtain $\vec{p}^{\,1}, \dots, \vec{p}^{\,\mu}$. Then, we place $\mu$ points (starting from $\vec{p}^{\,1}$) uniformly along the piecewise linear curve defined by connecting consecutive pairs: $(\vec{p}^{\,1}, \vec{p}^{\,2}), (\vec{p}^{\,2}, \vec{p}^{\,3}), \dots$. For $k > 2$ objectives, we apply Delaunay triangulation \cite{delaunay} to the set, and fill each triangle uniformly at random with a number of points proportional to its area.

After filling, we perform b), \textbf{reduction}, to return to exactly $\mu$ elements. This is achieved by applying $k$-means clustering to the filled set and selecting the centroids of the $\mu$ resulting clusters.

Finally, in step c), \textbf{shifting}, we ideally want to shift $Y_0$ orthogonally to the Pareto front, toward the Utopian region. This involves two steps: (i) computing a shift direction vector $\vec{\eta} \in \mathbb{R}^k$, and (ii) shifting each point in $Y_0$ by $\delta \vec{\eta}$, where $\delta > 0$, resulting in the reference set $R$.

To compute $\vec{\eta}$, we follow the procedure described by~\cite{WangRUCS24}, in which $\vec{\eta}$ is chosen to be orthogonal to the convex hull of the points $\vec{m}^{1}, \ldots, \vec{m}^{k}$. Each point $\vec{m}^{i} \in Y_0$ is defined as the point with the minimal value in the $i$-th objective. More precisely, we compute $\vec{\eta}$ as follows. Given $Y_0 = \{\vec{y}^{\;1}, \ldots, \vec{y}^{\,\mu}\}$, we define, for $i=1,\ldots, k$
\begin{equation}
    \vec{m}^{i} \coloneqq\vec{y}^{\,\ell}, \text{ such that } \vec{y}^{\,\ell}\in Y_0 \wedge \forall j\in[1..\mu], y^\ell_i \leq y^j_i\;,
\end{equation}
where $y^j_i$ denotes the $i$-th component of $\vec{y}^{\,j}$. Using these, we construct the matrix
\begin{equation}
 M := (\vec{m}^{2} - \vec{m}^{1}, \vec{m}^{3} - \vec{m}^{1}, \ldots, \vec{m}^{k} - \vec{m}^{1})\in\mathbb{R}^{k\times(k-1)}. 
\end{equation}
We then compute the QR factorization of $M$:
\begin{equation} \label{eq:M-matrix}
 M = Q\tilde{R} = (\vec{q}^{\,1},\ldots,\vec{q}^{\,k})\tilde{R},
\end{equation}
where $Q \in \mathbb{R}^{k \times k}$ is an orthogonal matrix with columns $\vec{q}_i$, and $\tilde{R} \in \mathbb{R}^{k \times (k - 1)}$ is an upper triangular matrix. The desired shifting direction is then given by
\begin{equation}
\label{eq:eta}
\vec{\eta} = -\text{sgn}(q_{1}^k) \frac{\vec{q}^{\,k}}{\|\vec{q}^{\,k}\|_2}.
\end{equation} 
Finally, the reference set $R =\{\vec{r}^{\,1}, \ldots, \vec{r}^{\,\mu}\}$ is obtained by shifting all points of $Y_0$ in the direction of $\vec{\eta}$ toward the utopian region, using the formula $\vec{r}^{\,i} = \vec{y}^{\,i} + \delta \vec{\eta}$ for $i = 1, \ldots, \mu$. Here, $\delta$ controls how far $Y_0$ is shifted.

\section{Implementation details} 
Each experiment case in~\cref{sec:experiments} is executed on an \texttt{Intel(R) Xeon(R) Gold 6126 CPU @ 2.60GHz} CPU (24 cores), which takes about 10 minutes to finish.

\subsection{Hessian preconditioning}\label{sec:preconditioning}
We took the Algorithm 3.3 from Nocedal's book~\cite{NocedalW99} to precondition the MMD Hessian $\nabla^2\MMD^2$, which modifies the diagonal of the Hessian, i.e., $\nabla^2\MMD^2 + \tau I, \tau >0$, and successively increases $\tau$ until $\nabla^2\MMD^2 + \tau I \succ 0$. We briefly recap it: given $\nabla^2\MMD^2$ we first perform Cholesky decomposition. If it fails (meaning it is indefinite), we set the initial $\tau=\max\{0, \beta - \min\{\operatorname{diag}(\nabla^2\MMD^2)\}\}$, where $\beta=10^{-6}$ and the operator $\operatorname{diag}$ takes the diagonal of a matrix. Then, we modify the Hessian to $\nabla^2\MMD^2 + \tau I$ and perform the Cholesky decomposition to check if it is positive-definite. If not, we double the $\tau$ value and perform the Cholesky step again. This procedure terminates whenever the Cholesky decomposition succeeds, which indicates $\nabla^2\MMD^2 + \tau I \succ 0$.

\subsection{Determine length-scale of the kernel} \label{sec:hyperparameter-setting}
For each test problem in~\cref{sec:experiments}, we propose a simple heuristic to determine the kernel function $k$ and the length-scale $\theta$ of it: since the Hessian matrix $\nabla^2 \MMD^2(Y,R;k,\theta)$ is often indefinite, we first precondition the Hessian and compute the condition number $\kappa$ for a choice of $k$ and $\theta$, and then choose the $(k, \theta)$ pair that minimizes the condition number, i.e., 
$$
\theta^* = \argmin\limits_{k, \theta} \kappa\left(\nabla^2 \MMD^2(Y,R;k,\theta) + \tau I\right),
$$
where $\tau >0$ is the preconditioning constant determined by the procedure in~\cref{sec:preconditioning}. Practically, instead of solving the above optimization problem, we only consider $k\in\{\text{Gaussian},\text{Laplace}\}$ and $\theta\in\{10^{-2}, 10^{-1}, 1, 10, 100, 500, 1\,000, 5\,000\}$. A simple grid search can determine the best choice, which we report in~\cref{tab:hyperparameter}.
\begin{table}[ht]
\setlength{\tabcolsep}{2.7pt}
\caption{The choice of kernel type and the length-scale $\theta$ value for each problem and baseline MOEA.\label{tab:hyperparameter}}
\centering
\small
\begin{minipage}{0.49\textwidth}
\centering
\begin{tabular}{lllr}
\toprule
Baseline & Problem & Kernel & $\theta$ \\
\midrule
NSGA-II & ZDT1 & Gaussian & 2\,000 \\
NSGA-II & ZDT2 & Gaussian & 2\,000 \\
NSGA-II & ZDT3 & Gaussian & 2\,000 \\
NSGA-II & ZDT4 & Laplace & 1 \\
NSGA-II & DTLZ1 & Laplace & 500\\
NSGA-II & DTLZ2 & Laplace & 500 \\
NSGA-II & DTLZ3 & Laplace & 500 \\
NSGA-II & DTLZ4 & Laplace & 500 \\
NSGA-II & DTLZ5 & Laplace & 500 \\
NSGA-II & DTLZ6 & Laplace & 500 \\
NSGA-II & DTLZ7 & Laplace & 500 \\ \midrule
NSGA-III & ZDT1 & Gaussian & 2\,000 \\
NSGA-III & ZDT2 & Gaussian & 2\,000 \\
NSGA-III & ZDT3 & Gaussian & 2\,000 \\
NSGA-III & ZDT4 & Laplace & 1 \\
NSGA-III & DTLZ1 & Laplace & 500 \\
\end{tabular}
\end{minipage}
\hspace{-12mm}
\begin{minipage}{0.49\textwidth}
\centering
\begin{tabular}{lllr}
NSGA-III & DTLZ2 & Laplace & 500 \\
NSGA-III & DTLZ3 & Laplace & 500 \\
NSGA-III & DTLZ4 & Laplace & 500 \\
NSGA-III & DTLZ5 & Laplace & 500 \\
NSGA-III & DTLZ6 & Laplace & 500 \\
NSGA-III & DTLZ7 & Laplace & 500 \\ \midrule
MOEA/D & ZDT1 & Laplace & 1 \\
MOEA/D & ZDT2 & Laplace & 1 \\
MOEA/D & ZDT3 & Laplace & 100 \\
MOEA/D & ZDT4 & Laplace & 1 \\
MOEA/D & DTLZ1 & Laplace & 500 \\
MOEA/D & DTLZ2 & Laplace & 500 \\
MOEA/D & DTLZ3 & Laplace & 500 \\
MOEA/D & DTLZ4 & Gaussian & 2\,000 \\
MOEA/D & DTLZ5 & Laplace & 1 \\
MOEA/D & DTLZ6 & Laplace & 500 \\
MOEA/D & DTLZ7 & Gaussian & 5\,000 \\
\bottomrule
\end{tabular}
\end{minipage}
\end{table}

\subsection{Estimating the CPU time of automatic differentiation} \label{sec:AD-CPU-time}
We measure the CPU time of function evaluations $F(\vec{x})$ and Jacobian/Hessian calls (with automatic differentiation) $\operatorname{D}\!F(\vec{x}),\operatorname{D}^2\!F(\vec{x})$ for each test problem in \cref{sec:experiments} on an \texttt{Intel(R) Xeon(R) Gold 6126 CPU @ 2.60GHz} CPU (24 cores). The problem implementation is taken from \texttt{Pymoo}\footnote{\url{https://pymoo.org/}} library~\cite{BlankD20} (Apache-2.0 License; version $0.6.1.1$) and the automatic differentiation is performed with the Python \texttt{JAX}\footnote{\url{https://docs.jax.dev/en/latest/quickstart.html}} library~\cite{jax2018github} (Apache-2.0 License; version $0.4.23$). Both the function evaluation and Jacobian/Hessian are repeated 100\,000 times with input $\vec{x}$ sampled u.a.r.~from the problem's domain. We report the summary statistics of the \emph{CPU times} in~\cref{tab:CPU_time}. Averaging the CPU time ratio of Jacobian to the function evaluations over all repetitions gives approximately a value of $1.47$, while the ratio of Hessian to the function evaluation is around $1.89$.
\begin{table}[ht]
\centering
\setlength{\tabcolsep}{5pt}
\def\arraystretch{1.2}
\caption{
The mean, median, and upper $90\%$ percentile of the CPU time measured (in microseconds) over 100\,000 repetitions of function evaluations, Jacobian calls, and Hessian calls for each problem.
\label{tab:CPU_time}
}
\begin{tabular}{lccccccccc}
\toprule
\multirow{2}{*}{Problem} & \multicolumn{3}{c}{Function evaluation ($\mu s$)} & \multicolumn{3}{c}{Jacobian call ($\mu s$)} & \multicolumn{3}{c}{Hessian call ($\mu s$)} \\ \cline{2-10}
 & mean & median & 90\% & mean & median & 90\% & mean & median & 90\% \\
\midrule
ZDT1 & 3.524965 & 3.0 & 4.0 & 5.039280 & 5.0 & 6.0 & 6.870129 & 7.0 & 8.0 \\
ZDT2 & 3.498975 & 3.0 & 4.0 & 5.285733 & 5.0 & 6.0 & 6.436024 & 6.0 & 7.0 \\
ZDT3 & 3.483405 & 3.0 & 4.0 & 5.379644 & 5.0 & 6.0 & 8.003440 & 8.0 & 9.0 \\
ZDT4 & 3.462605 & 3.0 & 4.0 & 5.594336 & 5.0 & 6.0 & 8.334733 & 8.0 & 9.0 \\
DTLZ1 & 4.189592 & 4.0 & 5.0 & 5.927459 & 6.0 & 7.0 & 8.029110 & 8.0 & 9.0 \\
DTLZ2 & 4.460365 & 4.0 & 5.0 & 6.198592 & 6.0 & 7.0 & 7.415784 & 7.0 & 8.0 \\
DTLZ3 & 4.214652 & 4.0 & 5.0 & 6.240592 & 6.0 & 7.0 & 8.141171 & 8.0 & 9.0 \\
DTLZ4 & 4.159192 & 4.0 & 5.0 & 5.913939 & 6.0 & 7.0 & 7.383044 & 7.0 & 8.0 \\
DTLZ5 & 4.144941 & 4.0 & 5.0 & 6.206872 & 6.0 & 7.0 & 8.452215 & 8.0 & 9.0 \\
DTLZ6 & 4.252213 & 4.0 & 5.0 & 6.348753 & 6.0 & 7.0 & 8.859419 & 9.0 & 10.0 \\
DTLZ7 & 4.199972 & 4.0 & 5.0 & 5.864169 & 6.0 & 7.0 & 7.032360 & 7.0 & 8.0 \\
\bottomrule
\end{tabular}
\end{table}


\end{document}